\def\year{2022}\relax
\documentclass[letterpaper]{article} 
\usepackage{aaai22}  
\usepackage{amsmath}  
\usepackage{amsthm}
\usepackage{times}  
\usepackage{helvet}  
\usepackage{courier}  
\usepackage[hyphens]{url}  
\usepackage{graphicx} 
\usepackage{natbib}  
\usepackage{caption} 
\DeclareCaptionStyle{ruled}{labelfont=normalfont,labelsep=colon,strut=off} 
\usepackage{booktabs} 
\usepackage{xcolor}
\definecolorset{gray/rgb/hsb/cmyk}{}{}%
 {black,0/0,0,0/0,0,0/0,0,0,1;%
  darkgray,.25/.25,.25,.25/0,0,.25/0,0,0,.75;%
  gray,.5/.5,.5,.5/0,0,.5/0,0,0,.5;%
  lightgray,.75/.75,.75,.75/0,0,.75/0,0,0,.25;%
  white,1/1,1,1/0,0,1/0,0,0,0}

\usepackage{algorithm}
\usepackage{algorithmic}

\usepackage{amsfonts}

\newtheorem{definition}{Definition}

\newtheorem{theorem}{Theorem}
\newtheorem{remark}{Remark}

\usepackage{newfloat}
\usepackage{listings}
\floatstyle{ruled}
\newfloat{listing}{tb}{lst}{}
\floatname{listing}{Listing}
\title{Interpretable Clustering via Multi-Polytope Machines}
 \author {
     Connor Lawless\textsuperscript{\rm 1}\thanks{Work done while an intern at IBM Research.},
     Jayant Kalagnanam\textsuperscript{\rm 2},
     Lam M. Nguyen\textsuperscript{\rm 2},
     Dzung Phan\textsuperscript{\rm 2},
     Chandra Reddy\textsuperscript{\rm 2}
 }
 \affiliations {
     \textsuperscript{\rm 1} Cornell Univeristy, Operations Research and Information Engineering, Ithaca, NY 14850\\
     \textsuperscript{\rm 2} IBM Research, Thomas J. Watson Research Center, Yorktown Heights, NY 10598, USA\\
     cal379@cornell.edu, jayant@us.ibm.com, LamNguyen.MTLD@ibm.com, phandu@us.ibm.com, creddy@ibm.com
}

\usepackage{bibentry}

\begin{document}

\maketitle

\begin{abstract}
Clustering is a popular unsupervised learning tool often used to discover groups within a larger population such as customer segments, or patient subtypes. However, despite its use as a tool for subgroup discovery and description - few state-of-the-art algorithms provide any rationale or description behind the clusters found. We propose a novel approach for interpretable clustering that both clusters data points and constructs polytopes around the discovered clusters to explain them. Our framework allows for additional constraints on the polytopes - including ensuring that the hyperplanes constructing the polytope are axis-parallel or sparse with integer coefficients. We formulate the problem of constructing clusters via polytopes as a Mixed-Integer Non-Linear Program (MINLP). To solve our formulation we propose a two phase approach where we first initialize clusters and polytopes using alternating minimization, and then use coordinate descent to boost clustering performance. We benchmark our approach on a suite of synthetic and real world clustering problems, where our algorithm outperforms state of the art interpretable and non-interpretable clustering algorithms.\end{abstract}

\noindent 
\section{Introduction}

Clustering is an unsupervised machine learning problem that aims to partition unlabelled data into groups. In practice, it is often used as a tool for discovering sub-populations within a dataset such as customer segments, disease heterogeneity, or movie genres. In these applications the group assignment itself is often of secondary importance to the interpretation of the groups found. However, traditional clustering algorithms simply output a set of cluster assignments and provide no explanation or interpretation for the discovered groups. This fundamental misalignment between algorithms and applications force practitioners to work backwards from cluster assignments to fit post-hoc explanations to the output of traditional clustering algorithms. Motivated by these shortcomings, the field of interpretable clustering aims to jointly cluster points and provide an explanation of the groups themselves. 

Recent work on interpretable clustering has focused on leveraging decision trees, a popular interpretable supervised learning tool, to construct clusters \cite{ghattas2017clustering, pmlr-v119-moshkovitz20a, frost2020exkmc, bertsimas2017optimal, fraiman2013interpretable, basak2005interpretable, jin2001cell, de1997using, yasami2010novel}. While these approaches have intuitive appeal, as decision trees are a well studied model class that are easy to understand, they focus primarily on clusters that can be defined by axis-parallel partitions of the feature space. In contrast, interpretable supervised learning has a much wider range of tools available to practitioners including sparse linear models \cite{tibshirani1996regression, ustun2013supersparse, bertsimas2020sparse}, rule sets \cite{dash2018boolean, wang2015learning, wang2017bayesian, lawless2021fair}, or score cards \cite{ustun2017optimized} amongst others. In this paper we introduce a novel method for interpretable clustering that provides cluster explanations by constructing polytopes around each cluster. Our framework allows for constraints on the hyperplanes that construct each polytope allowing for a wider range of cluster explanations including axis parallel partitions (similar to decision trees), partitions defined by sparse integer hyperplanes (similar to sparse/integer models), and general linear models (similar to SVMs). Figure \ref{fig:mpc_ex} shows an example of a polytope cluster explanation under the three different constraints. Overall our approach provides practitioners more flexibility to explain clusters inline with the business requirements of their application.

\begin{figure*}[!htb]
    \centering
      \includegraphics[width=1\textwidth]{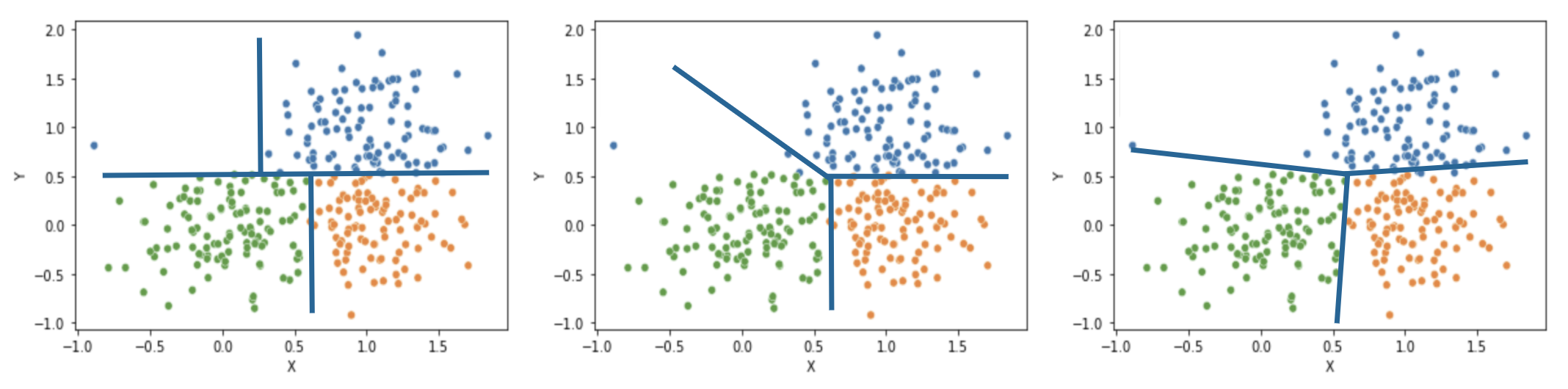}
  \caption {\label{fig:mpc_ex} Sample clusters explained by polytopes with different constraints on the hyperplanes. (Left) Polytopes are composed of axis-parallel hyperplanes giving rise to rectangular cluster explanations. (Center) Polytopes are composed with integral hyperplanes allowing only diagonal or axis parallel lines. (Right) Polytopes composed with general hyperplanes. }
\end{figure*}

\subsection{Related Work}
Existing interpretable clustering methods can be grouped into two general approaches: post-hoc explanations and integrated interpretation and clustering. Post-hoc approaches take the output of any clustering algorithm and attempt to fit an explanation to it. A common heuristic approach is to simply use an interpretable supervised learning algorithm to predict the cluster labels \cite{jain1999data, de2017explaining, kauffmann2019clustering}. Recent work has also looked at designing specialized decision tree algorithms to explain the output of a clustering algorithm using a new splitting criterion \cite{pmlr-v119-moshkovitz20a, frost2020exkmc}. Another post-hoc approach is to find a representative summary point, or prototype, for each cluster. Common choices include looking at summary statistics for each feature such as the mean, median or mode. Carrisoza et al. present an integer programming (IP) formulation for finding an optimal prototype and radius around the point that captures the cluster, trading off false positives and false negatives \cite{carrizosainterpreting}. While these post-hoc approaches work with any clustering algorithm, they fail to incorporate the down-stream interpretation task in the generation of the clusters themselves. 

In response to the shortcomings of post-hoc approaches, recent work has focused on integrated approaches that perform clustering with the interpretation task in mind. Liu et al. transform clustering into a supervised learning problem by augmenting the dataset with synthetic data points and use a decision tree algorithm to classify the original points from the synthetic, resulting in a decision tree with leaf nodes representing clusters \cite{liu2000}. Researchers have also adapted existing heuristic decision tree approaches to perform clustering by using new splitting criterion \cite{fraiman2013interpretable}. Most similar to our work, Bertsimas et al. formulate the problem of finding an optimal decision tree to perform clustering as a mixed integer optimization problem and construct an approximate solution via coordinate descent \cite{bertsimas2021interpretable}. In a similar vein, recent work has looked at building rule sets to explain clusters \cite{chen2016interpretable, chen2018interpretable, carrizosa2021clustering, pelleg2001mixtures}. In contrast to the existing state of the art, our approach is the first to look at a more general function class to explain clusters. Our approach has more expressive power than decision-tree or rectangular approaches as polytopes with axis-parallel hyperplanes can be mapped to a decision tree. We also provide more flexibility to practitioners, allowing them to trade off the relative importance of interpretability and cluster quality.

\subsection{Main Contributions}
We summarize our main contributions as follows:
\begin{itemize}
    \item We propose a novel mixed integer optimization non-linear programming (MINLP) formulation for interpretable clustering that jointly clusters points and constructs polytopes surrounding each cluster.
    \item As a component of our MINLP framework we introduce a representation aware $k$-means clustering formulation that forms clusters with interpretability integrated into the objective.
    \item We present a formulation to find separating hyperplanes with sparse integer coefficients for interpretability.
    \item To approximate the solution of the MINLP formulation we introduce a two-stage optimization procedure that initializes clusters via alternating minimization and then optimizes the Silhouette coefficient via coordinate descent.
    \item Numerical experiments on both synthetic and real-world datasets show our approach outperforms state of the art interpretable and uninterpretable clustering algorithms.
\end{itemize}

The remainder of this paper is organized as follows. Section \ref{sec:mio} details the MINLP formulation for polytope clustering. Section \ref{sec:algo} gives an overview of our two-stage optimization procedure to approximate the optimal solution to our formulation. Finally, we present a suite of numerical experiments on synthetic and real world datasets in Section \ref{sec:numerics}. 

\section{Mixed Integer Optimization Framework} \label{sec:mio}

In the standard clustering setting we are given a set of unlabelled data points ${\cal D} = \{x^t \in \mathbb{R}^D \}_{t=1}^{N}$ and asked to partition them into a set of $K$ clusters $C_1, \dots, C_K$, where  $C_i$ is the set of points belonging to cluster $i$. Note that we assume the data to have real-valued features and have an associated metric for defining pairwise distance between points. This is not a restrictive assumption as categorical data can be converted to real valued features via a standard one-hot encoding scheme. In the interpretable clustering setting we also have to provide an explanation of each cluster.

The key idea of our approach is to explain each cluster by constructing a polytope around it. To construct such a polytope we find a separating hyperplane for each pair of clusters. The intersection of the half-spaces generated by the set of hyperplanes involving each cluster then defines the polytope. In this section we present a mixed integer optimization formulation for jointly finding clusters and defining polytopes. We start by considering the interpretation and clustering problems separately before joining them in a unified framework.   

\subsection{Interpretable Separating Hyperplanes} \label{sec:sep_hyp}

Consider the setting where we have a fixed set of cluster assignments and need to construct a hyperplane to separate a pair of clusters $C_i$ and $C_j$. 
To construct a polytope to describe cluster $C_i$ we would need to construct one such hyperplane between $C_i$ and every other cluster. This problem bares a striking resemblance to the classic support vector machine (SVM) problem \cite{boser1992training}, where the goal is to find a hyperplane that separates two classes of data with a maximum margin. However, the standard SVM approach adds no additional constraints on the resulting hyperplane - potentially leading to dense hyperplanes with decimal values that are difficult for end users to interpret. We instead use a novel IP formulation that constructs a separating hyperplane with small integer coefficients and a limit on the number of non-zero coefficients.

Let $w^{ij}$ and $b^{ij}$ be the slope and intercept of the separating hyperplane between clusters $i$ and $j$. Furthermore let $w^{ij}_{d,+}$ and $w^{ij}_{d,-}$ represent non-negative components of element $d$ in $w$ - namely $w^{ij}_d = w^{ij}_{d,+} - w^{ij}_{d,-}$. We also introduce a constant $M$ that represents the maximum allowable integer coefficient value. To add constraints on the sparsity of the hyperplane we introduce binary variables $y^{ij}_{d,+}$ and $y^{ij}_{d,-}$ that track whether feature $d$ is included in the hyperplane. We put a hard constraint of $\beta$ on the number of non-zero coefficients in the final hyperplane. $\xi^{ij}_t$ tracks the mis-classification of data point $t$ - specifically its distance from the correct side of the hyperplane to be classified correctly.  Finally, $\epsilon$ is a fixed constant for the minimum non-zero separation distance with respect to any feasible hyperplane (i.e., the smallest non-zero distance between two points with respect to feasible hyperplane). With this notation in mind, the IP formulation for constructing an interpretable separating hyperplane is the following:

    \begin{align}
    &\min_{w, b, \xi} &\sum_{t \in C_i \cup C_j} (\xi^{ij})_t  \label{obj:sep_hyp}\\
    &\textbf{s.t.}~~ &(w^{ij})^T x^t + b^{ij} &\geq -(\xi^{ij})_t, \ \forall x^t \in C_i \label{const:w_ci}\\
    &&(w^{ij})^T x^t + b^{ij} + \epsilon &\leq (\xi^{ij})_t, \ \forall x^t \in C_j \label{const:w_cj} \\
    && w^{ij}_{d,+} - w^{ij}_{d,-} &=  w^{ij}_d ~~~ \forall d \in [D] \label{const:w_def} \\
    && \sum_{d \in [D]} w^{ij}_{d,+} + w^{ij}_{d,-} &\geq 1 ~~\label{const:w_trivial}\\
    && y^{ij}_{d,+} + y^{ij}_{d,-} &\leq 1 ~~~\forall d \in [D] \label{const:y_trival}\\
    &&  \sum_{d \in [D]} y^{ij}_{d,+} + y^{ij}_{d,-}&\leq \beta  \label{const:sparsity}\\
    && 0 \leq w^{ij}_{d,+} &\leq My^{ij}_{d,+} ~~~\forall d \in [D] \label{const:M+}\\
    && 0 \leq w^{ij}_{d,-} &\leq My^{ij}_{d,-} ~~~\forall d \in [D]\label{const:M-}\\
    &&w^{ij} \in {\cal Z}^d,& w^{ij}_{+}, w^{ij}_{-} \in {\cal Z}_{\geq 0}^d \\
    && &(\xi^{ij})_t \geq 0 \\
    && &y^{ij}_d \in \{0,1\}. \label{const:y_bin}
    \end{align}

The objective (\ref{obj:sep_hyp}) is simply to minimize the classification error of the separating hyperplane. Constraints (\ref{const:w_ci}) and (\ref{const:w_cj}) track whether we are classifying data point $x^t$ correctly. Note that we add $\epsilon$ to constraint (\ref{const:w_cj}) to ensure that the separating hyperplane is only inclusive of cluster $C_i$. In other words, a data point in cluster $C_j$ is only classified correctly if it lies strictly below the hyperplane. Constraints (\ref{const:w_trivial}) and (\ref{const:y_trival}) ensure that the trivial hyperplane (all zero) is excluded by ensuring that the $\ell_1$ norm of the hyperplane is above 1 (the smallest possible $\ell_1$ norm of an integer hyperplane) and that only at most one of $w^{ij}_{d,+}$ and $w^{ij}_{d,-}$ are non-zero. Constraint (\ref{const:sparsity}) bounds the number of non-zero coefficients. Finally constraints (\ref{const:M+}) and (\ref{const:M-}) constrain the maximum integer values of the coefficients. One interpretation of the constant $M$ is that it controls the search space of possible hyperplanes. Start by noting that any general hyperplane can be normalized to have coefficients between $-1$ and $1$ by dividing by the largest coefficient. For integer hyperplanes with maximum value $M$ normalizing by $M$ gives possible coefficient values $\frac{n}{M}$ where $n$ is an integer between $-M$ and $M$. Thus increasing $M$ grows the number of feasible hyperplanes.

Note that when solving this problem we only consider data points in $C_i$ and $C_j$, which in settings with a large number of clusters can be substantially less than the full data set. This allows our IP formulation to scale to larger datasets while limiting the computational burden of solving each individual IP. 


\subsection{Silhouette Clustering with Cardinality Constraints} \label{sec:sil_cluster}

 Now consider the problem of finding cluster assignments. High quality clusters are generally defined by having low \emph{intra}-cluster distance (i.e. the distance between points in the same cluster), and high \emph{inter}-cluster distance (i.e. distance between points in different clusters). There are a number of cluster quality metrics that incorporate this high-level concept, one of the most popular being silhouette coefficient. The silhouette coefficient uses the average distance between a data point $t$ and all other data in the same cluster as a measure of intra-cluster distance, and the average distance between $t$ and every point in the second closest cluster as a measure of inter-cluster distance. 

\begin{definition}[Silhouette Coefficient] Consider data point $t$ with cluster label k. Let $r(t)$ be the average distance between data point $t$ and every other point in the same cluster. Let $q(t)$ be the average distance between data point $t$ and every point in the second closest cluster. The silhouette score for data point $t$ is defined as:
\begin{align*}
r(t) = \frac{1}{|C_k|-1}\sum_{j \in C_k}d_{tj} \\
q(t) = \min_{l = 1, \dots, K: l \neq k} \frac{1}{|C_l|}\sum_{j \in C_l}d_{tj} \\
s(t) = \frac{q(t) - r(t)}{\max(q(t),r(t))}
\end{align*}

The silhouette score for a set of cluster assignments is the average of the silhouette scores for all the data points. The possible values range from -1 (worst) to +1 (best).
\end{definition}


Similar to \cite{bertsimas2021interpretable} we now formulate the silhouette clustering problem as a MINLP. Let $z_{tk}$ be the binary variable indicating whether data point $t$ is assigned to cluster $k$, and the variables $u_k$ be the binary variable indicating whether cluster $k$ is used. To track the silhouette coefficient, let $s_t$ be the silhouette score for data point $t$, $q_t$ represent the intercluster distance measure $q(t)$, and $r_t$ represent the intracluster distance measure $r(t)$. Let $c_{tk}$ track the distance from data point $t$ to cluster $k$, and $\gamma_{tk}$ be the binary variable indicating the second closest cluster for data point $t$. Using this notation, the formulation for finding the cluster assignments is:

\begin{align}
    &\min_{z, c_k, u_k} & -\frac{1}{n}&\sum_{t \in {\cal D}} s_t \label{obj:cluster_qual} \\
    &\textbf{s.t.} &\sum_{k = 1}^{K} z_{tk} &= 1 ~\forall t\in \mathcal{D} \label{const:assignment}\\
        &&N_{min}u_k \leq \sum_{t\in\mathcal{D}}z_{tk} &\leq N_{max}u_k, ~\forall k = 1, \dots, K  \label{const:cardinality}\\
    && s_t &= \frac{q_t - r_t}{m_t} ~~ \forall t \in {\cal D} \label{const:sil1} \\
    && c_{tk} &= \frac{1}{N_k} \sum_{j \in {\cal D}} d_{ij}z_{jk}, ~ \forall t, k \\
    && N_k &= \sum_{t \in {\cal D}} z_{tk} ~~ \forall k \in [K]\\
    && r_t &= \sum_k c_{tk}z_{tk} ~~ \forall t \in {\cal D}\\
    && q_t &\geq \sum_k \gamma_{tk} c_{tk} ~~ \forall t \in {\cal D}\\
    && \sum_{k} \gamma_{tk} &= 1 ~~ \forall t \in {\cal D} \\
    && \gamma_{tk} &\leq 1 - z_{tk}  ~~ \forall t \in {\cal D}\\
    && m_t &\geq r_t, q_t \label{const:sil2}\\
    &&u_k, z_k &\in \{0,1\}.\label{const:int_clustering}
\end{align}

The objective of the formulation is to maximize the silhouette coefficient. Constraint (\ref{const:assignment}) ensures that we assign each data point to at most one cluster. 
Constraint (\ref{const:cardinality}) sets cardinality constraints (i.e., minimum and maximum values: $N_{min}$ and $N_{max}$) on the size of the clusters. Finally, constraints (\ref{const:sil1})-(\ref{const:sil2}) track the silhouette coefficient for each data point $t$. Note that this formulation is a mixed integer non-linear program with non-linearity introduced by the silhouette coefficient.

\subsection{Joint Optimization Framework}

The goal of the joint framework for clustering and polytope construction is to both maximize cluster quality, defined by the silhouette coefficient, and minimize the representation error. To control the relative importance given to representation error and cluster quality we also introduce a regularization parameter $\lambda$. We define the representation error as the sum of the mis-classification costs for all the hyperplanes defining each cluster's polytope. Given a hyperplane between two clusters $C_i, C_j$, we can determine the hypothetical mis-classification error for every data point (including those not in $C_i$ and $C_j$) if they were assigned to $C_i$ or $C_j$ by computing the distance from the point to the hyperplane. Let $(\xi_+^{ij})_t$ and $(\xi_-^{ij})_t$  be the error for data point $t$ if it is assigned to cluster $i$ and $j$ respectively. Let $ {\cal K} \times {\cal K}$ be the set of all pairs of clusters. Combining the hyperplane and clustering formulations, the joint framework for constructing polytope clusters is:  
\begin{align}
    &\min_{z, c_k, u_k} -\frac{1}{n}\sum_{t \in {\cal D}} s_t~~+ \label{obj:sil_error}\\
    &\lambda\sum_{x^t \in \mathcal{D}} \sum_{i=1}^{K-1}\sum_{j=i+1}^{K} (z_{tk}(\xi_+^{ij})_t+z_{tj}(\xi_{-}^{ij})_t) \label{obj:rep_error}\\
    &\textbf{s.t. (\ref{const:assignment})-(\ref{const:int_clustering})}\\
    &\quad~~\textbf{(\ref{const:w_def})-(\ref{const:y_bin}) for $i,j  \in  {\cal K} \times {\cal K}$ }\\
    &(w^{ij})^T x^t + b^{ij} \geq -(\xi^{ij}_+)_t ~~\forall i,j  \in  {\cal K} \times {\cal K}, t \in {\cal D} \label{const:w_ci_z}\\
    &(w^{ij})^T x^t + b^{ij} + \epsilon \leq (\xi^{ij}_{-})_t ~~\forall i,j  \in  {\cal K} \times {\cal K}, t \in {\cal D}  \label{const:w_cj_z} 
\end{align}
Note that the objective now trades off cluster quality, i.e., Eq. (\ref{obj:sil_error}), and representation error, i.e., Eq. (\ref{obj:rep_error}). We also need one separating hyperplane problem per pair of clusters.

\begin{remark} \label{rem:features} There is no requirement that the feature space for clustering and hyperplane separation be the same. For instance real valued features can be used to compute the cluster quality, but the separation is done in a binarized feature space to ensure more interpretable explanations.
\end{remark}

\section{Algorithm Overview} \label{sec:algo}
The joint formulation for polytope clustering is a MILNP, and thus is difficult to optimize globally. Instead, we use a two-stage procedure to find a high quality approximation of the solution. In the first stage we use alternating minimization to find an initial set of clusters and separating hyperplanes. We then use coordinate descent to improve the clustering performance of our assignments and explanation.

\subsection{Initialization via Alternating Minimization}
We decompose the joint framework (\ref{obj:sil_error})-(\ref{const:w_cj_z}) into two components: cluster assignments, and constructing hyperplanes to separate clusters. The key intuition behind our initialization procedure is we alternate between clustering the points into $K$ clusters $C_1, C_2, \dots C_K$ and then constructing interpretable separating hyperplanes between each pair of clusters. The silhouette clustering problem as presented in Section \ref{sec:sil_cluster} remains a difficult problem to solve due to the non-linearity presented by the silhouette metric. For computational tractability, we instead use the $k$-means clustering objective \cite{jain1999data} as a proxy for silhouette coefficient during the initialization phase. This leads to the following much simpler formulation for determining cluster assignments:
 
\begin{align}
    &\min_{z, c_k, u_k} &\hspace{-10mm}\sum_{k=1}^{K} \sum_{x^t \in \mathcal{D}} &z_{tk} \|x^t - c_k\|^2 + \nonumber\\
    &&\hspace{-10mm}\lambda\sum_{x^t \in \mathcal{D}} \sum_{i=1}^{K-1}\sum_{j=i+1}^{K}& (z_{tk}(\xi_+^{ij})_t+z_{tj}(\xi_{-}^{ij})_t) ~~ \label{obj:kmeansrep}\\
    &\textbf{s.t. (\ref{const:assignment})-(\ref{const:cardinality})}  \label{consts:kmeansrep}
\end{align}

Note that this problem is now similar to a traditional $k$-means clustering problem, with an additional objective term to capture representation error. We denote this new clustering formulation the representation aware $k$-means clustering problem. To solve this formulation for a fixed set of representation errors $\xi$ we alternate between fixing the cluster centers $c_k$ and generating assignments by solving the IP, and then fix the assignments and update the cluster centers and repeat the process until convergence. To begin the initialization procedure we solve the cluster assignment with no representation errors (i.e., $\xi_t = 0 ~\forall t$), and then generate polytopes by solving the separating hyperplane problem detailed in Section \ref{sec:sep_hyp} for every pair of clusters. The representation errors generated by the separating hyperplane problems are then used in the representation aware $k$-means problem and the process is repeated until convergence. The entire initialization procedure is outlined in Algorithm \ref{alg:mpc_init}.

The choice of $k$-means as a proxy for silhouette coefficient is motivated both by its relative ease of optimization and the fact that $k$-means generally performs well as a proxy for silhouette coefficient (see experimental results in \cite{bertsimas2021interpretable}). In the absence of interpretability constraints on the separating hyperplanes, all local solutions of the $k$-means clustering problem also have the appealing property that they can be perfectly explained by a polytope. 
\begin{theorem}[$k$-means Polytope Interpretability]
 Local solutions to the $k$-means clustering problem with Euclidean distance can be perfectly separated from the other clusters by a polytope.
\end{theorem}
Unfortunately this result does not hold with the addition of interpretability constraints. However the following theorem shows that our initialization procedure is still guaranteed to converge to a solution in a finite number of iterations.

\begin{theorem}[Alternating Minimization Improvement]
Algorithm \ref{alg:mpc_init} generates objective values for the representation aware $k$-means clustering problem that are monotonically decreasing for $l \geq 2$, and terminates in a finite number of iterations.
\end{theorem}

\begin{algorithm}[tb]
\caption{Cluster Initialization via Alternating Minimization}
\label{alg:mpc_init}
\textbf{Inputs}: Data ${\cal D}$, initial number of clusters $K$, $\lambda \geq 0$, $M \in {\mathbb Z} > 1$, $\beta \in {\mathbb Z} > 1, \epsilon > 0$ \\
\textbf{Output}: Cluster assignments $z \in \{0,1\}^{n \times K}$, separating hyperplanes $w, b$ 
\begin{algorithmic}[1] 
\STATE Set $(\xi_+^{ij})_t = (\xi_-^{ij})_t = 0 ~ \forall i,j, t$
\STATE Initialize $z_{tk}, c_k$ using conventional $k$-means algorithm on ${\cal D}$
\FOR{$l=1,2,\dots$}
\STATE /* Compute Cluster Assignments */
\FOR{$m=1,2,\dots$}
\STATE Fix $c_k$. Solve (\ref{obj:kmeansrep})-(\ref{consts:kmeansrep}) for updated  $z_{tk}$
\STATE Set $c_k = \frac{1}{N_k} \sum_{t \in {\cal D}} z_{tk} x^t$, $N_k = \sum_{t \in {\cal D}} z_{tk}$.
\ENDFOR
\STATE Set $C_i = \{x^t \in {\cal D} : z_{ti} = 1\} ~ \forall i=1,\dots, K$
\STATE /* Compute Separating Hyperplanes */
\FOR{i=1,\dots,K-1}
\FOR{j=i+1,\dots,K}
\STATE Solve (\ref{obj:sep_hyp})-(\ref{const:y_bin}) with $M, \beta, C_i, C_j$ for $w^{ij}, b^{ij}$
\STATE Set $(\xi^{ij}_+)_t = \max(-w^{ij}x^t + b,0) ~ \forall t \in {\cal D}$ 
\STATE Set $(\xi^{ij}_-)_t = \max(w^{ij}x^t + b + \epsilon, 0) ~ \forall t \in {\cal D}$ 
\ENDFOR
\ENDFOR
\ENDFOR
\STATE \textbf{return} z, w, b
\end{algorithmic}
\end{algorithm}

\subsection{Coordinate Descent}
Once an initial clustering and polytope explanation are in place we use a local search procedure to optimize the original clustering objective. The key idea behind the local search is we consider each polytope and try to adjust it to boost clustering performance. We consider the following local search operations:

\begin{itemize}
    \item \textbf{Boundary Shift}: For a given hyperplane we alter the slope and the intercept of the hyperplane and change cluster assignments based on the new boundary. Intuitively this can be thought of as shifting one boundary of the defining polytope for a cluster. Here $M$ and $\alpha$ restrict the search space of potential hyperplanes to consider, making an exhaustive consideration of potential hyperplanes feasible for small $M$ and $\alpha$.
    \item \textbf{Cluster Splitting}: For a given cluster we attempt to add a new hyperplane to split the cluster into two smaller clusters. Here we consider any feasible separating hyperplane (i.e., in accordance with $M$ and $\alpha$).
    \item \textbf{Cluster Merging}: For two adjacent clusters (i.e., two clusters separated only by one hyperplane), we attempt to remove that hyperplane and merge the clusters. 
\end{itemize}

We consider each local search operation and retain the cluster assignment with the best objective value. One of the properties of this coordinate descent approach is that it can increase or decrease the number of clusters present in the assignment through merging or splitting clusters. This allows our approach to be less sensitive to the initial number of clusters specified during the initialization procedure. To provide a fair comparison to algorithms that have a fixed number of clusters and to support applications where there is a constraint on the number of clusters desired, our coordinate descent procedure also puts an upper bound on the total number of possible clusters that can be generated. The entire clustering algorithm including coordinate descent is outlined in Algorithm \ref{alg:mpc}.

\begin{algorithm}[tb]
\caption{Multi-Polytope Clustering (MPC) Algorithm}
\label{alg:mpc}
\textbf{Input}: Data ${\cal D}$, initial number of clusters $K$, maximum cluster number $K_{\max}$, $\lambda \geq 0$, $M \in {\mathbb Z} > 1$, $\beta \in {\mathbb Z} > 1$ \\\
\textbf{Output}: Cluster assignments $z \in \{0,1\}^{n \times K}$, separating hyperplanes $w, b$ 
\begin{algorithmic}[1] 
\STATE Initialize $z, w, b$ using Algorithm \ref{alg:mpc_init}
\STATE Initialize processing queue ${\cal Q}$ with all hyperplane indices ($(i, j) ~ \forall i = 1, \dots, K-1, j = i+1, \dots, K$) and cluster indices ($i = 1, \dots, K$)
\STATE Compute current loss $\ell$ using (\ref{obj:kmeansrep}) with cluster assignment $z$  
\WHILE{${\cal Q}$  is not empty}
\FOR{$ q \in {\cal Q}$}
\IF{$q$ corresponds to a hyperplane $(i,j)$}
\STATE Find best new hyperplane between $i,j$
\ELSIF{$q$ corresponds to cluster $C_i$} 
\STATE Find best split for cluster $C_i$
\ENDIF
\STATE Compute updated loss $\ell'$  using (\ref{obj:kmeansrep})
\IF{$\ell' < \ell$}
\STATE Update $z$, $\ell = \ell', w, b$ 
\STATE Reset ${\cal Q}$ with all hyperplanes and clusters
\ENDIF 
\ENDFOR 
\ENDWHILE
\STATE \textbf{return} $z, w, b$
\end{algorithmic}
\end{algorithm}

\begin{remark} Our optimization procedure only uses the choice of clustering metric (i.e., silhouette coefficient) during coordinate descent. This means that our framework can be easily extended to other clustering metrics such as Dunn Index.
\end{remark}

\section{Numerical Results} \label{sec:numerics}
To showcase the performance of the MPC algorithm we present two results: (1) clustering performance comparison to state of the art clustering algorithms on synthetic and real world datasets, (2) a view of sample cluster explanations under different settings of our hyperparameters.

\subsection{Clustering Performance}
We ran the MPC algorithm under two sets of hyperparameters to represent different levels of interpretability. The first, MPC-1, sets $M=\beta=1$ and represents cluster explanations with only axis-parallel hyperplanes, providing a fair comparison to univariate decision tree based methods. The second, MPC-2, sets $M=3, \beta=2$ which allows for more general hyperplanes with up to two non-zero integer coefficients and coefficients within $[-3,3]$. To benchmark the performance of MPC we compared it to the following suite of traditional and interpretable clustering algorithms: $k$-means++ \cite{arthur2006k}, Gaussian Mixture Models (GMM) \cite{hastie2009unsupervised}, Hierarchical Clustering (HClust) \cite{hastie2009unsupervised}, Density-based Spatial Clustering of Applications with Noise (DBSCAN) \cite{ester1996density}, Interpretable Clustering via Optimal Decision Trees (ICOT) \cite{bertsimas2021interpretable}, and Explainable $k$-means Clustering (ExKMC) \cite{frost2020exkmc}. For the ExKMC algorithm we present two sets of results: one where we only allow 1 leaf node per cluster (ExKMC-1) and one with up to five leaf nodes per cluster (ExKMC-5). See the supplementary materials for a detailed description of all the benchmark algorithms and the implementation we used.

For all algorithms that require a specification of the number of clusters we tune $k$ between 2 and 10. We ran each algorithm 100 times with different random seeds and report the best result. 
We tested our algorithm on two sets of clustering datasets. The first is a set of 9 synthetic clustering instances called the fundamental clustering and projection suite \cite{ultsch2020fundamental}. The suite contains a range of problem sizes (212-4096 data points) with two or three real valued features. For every dataset we normalize all feature values to be between 0 and 1. A summary of the silhouette score for each algorithm on the synthetic instances is included in Table \ref{tab:fcps}. The MPC-2 algorithm is able to dominate existing clustering methods with respect to the silhouette coefficient, obtaining the best silhouette coefficient on all 9 datasets. Increasing interpretability, namely using MPC-1 instead of MPC-2, comes at a cost to performance with lower scores in a third of the datasets. However, MPC-1 is still able to match or outperform other interpretable algorithms (ICOT, ExKMC-1, ExKMC-5) on all datasets. 

To test the performance of MPC on a more realistic suite of clustering problems we present results for a set of 8 datasets from the UCI machine learning repository \cite{Dua:2019}. While the number of data points remains similar to the synthetic instances (150-4601 data points), these datasets have larger feature spaces (4 to 89 dimensions). These datasets also have a mix of data types including numeric and categorical features. For numeric features we normalized all values to be between 0 and 1, and for categorical features we apply a one hot encoding. A summary of the results on the UCI datasets is included in Table \ref{tab:uci}. Unlike the synthetic datasets, increasing the cardinality of the hyperplanes (i.e. using MPC-2 instead of MPC-1) has no impact on the performance of the algorithm. Both algorithms match or outperform the benchmark algorithms on all but one dataset (Framingham), where ICOT achieves the highest score. Additional experimental results including computation time and the optimal number of clusters for each approach are included in the supplementary materials.

\begin{table*}[h]
\centering\footnotesize
\caption{\label{tab:fcps} Silhouette score on fundamental clustering and projection suite. Asterisk indicates best performing algorithm on each data set.}
\setlength{\tabcolsep}{5pt} 
\begin{tabular}{l c c c c c c c c c c }		\toprule
Dataset & (n,d) & $k$-means++ & GMM & HClust & DBSCAN & ICOT & ExKMC-1 & ExKMC-5 & \textbf{MPC-1} & \textbf{MPC-2} \\\midrule
Atom & (800,2) & 0.615 & 0.613 & 0.601 & 0.525 &0.508 & 0.584 & 0.609& 0.601 & 0.617*\\
Chainlink & (1000,2) & 0.517 & 0.524 & 0.504 & 0.244 & 0.396 & 0.508 & 0.516& 0.518 & 0.519* \\
Engytime & (4096,2) & 0.438 & 0.422 & 0.410 & 0.439 &0.573* & 0.408 & 0.429 & 0.573* & 0.573* \\
Hepta & (212, 3) & 0.702* &0.702* &0.702* & 0.702* & 0.455 & 0.702* & 0.702*& 0.702* & 0.702* \\
Lsun & (400,2) & 0.558 & 0.549 & 0.535 & 0.530 & 0.562 & 0.557 & 0.558 & 0.568*& 0.568*\\
Target & (770,2) & 0.592 & 0.575 & 0.580 & 0.416 & 0.629*& 0.584 & 0.592& 0.629* & 0.629*\\
Tetra & (400,3) &0.506* & 0.506* & 0.495 & 0.506*& 0.504 & 0.506* & 0.506* & 0.506* & 0.506* \\
Two Diamonds& (800,2) & 0.631* & 0.631* & 0.631* & 0.631* & 0.486 & 0.631*& 0.631* & 0.631* & 0.631*\\
Wingnut & (1070,2) & 0.460* & 0.460* & 0.435 & 0.117 & 0.422 & 0.448 & 0.452 & 0.450 & 0.460*\\
\bottomrule
\end{tabular}%
\end{table*}%

\begin{table*}[h]
\centering\footnotesize
\caption{\label{tab:uci} Silhouette score on UCI Machine Learning Clustering problem sets. Asterisk indicates best performing algorithm on each data set.}
\setlength{\tabcolsep}{5pt} 
\begin{tabular}{l c c c c c c c c c c }		\toprule
Dataset & (n,d) &$k$-means++ & GMM & HClust & DBSCAN & ICOT & ExKMC-1 & ExKMC-5 & \textbf{MPC-1} & \textbf{MPC-2} \\\midrule
Iris & (150,4) & 0.629* & 0.629* & 0.629*& 0.629*& 0.629* & 0.629* & 0.629* & 0.629* & 0.629*\\
Wine & (178,12) & 0.381 & 0.386* & 0.386&  0.357 & 0.386* & 0.357 & 0.381 & 0.386* & 0.386*\\
Zoo & (101, 16) &0.409 & 0.395 & 0.416* & 0.405 & 0.416* & 0.396 & 0.409 & 0.416* & 0.416*\\
Seeds & (210, 6) &0.505 & 0.478 & 0.493  & 0.165 &  0.243 & 0.500 &  0.505 & 0.506* & 0.506*\\
Libras & (360, 89) &0.245* & 0.227 & 0.230  & 0.164 & 0.154 & 0.182 & 0.228 & 0.245* & 0.245*\\
Framingham & (3658, 14) &0.405 & 0.371 & 0.380  & 0.189 & 0.454*& 0.403 & 0.405 & 0.406 & 0.406\\
Bank & (4521, 50) & 0.124 & 0.079 & 0.120  & 0.046 & 0.113 & 0.122 & 0.124 & 0.125* & 0.125*\\
Spam & (4601, 56) &0.677 & 0.524 & 0.602  & 0.546 & 0.760* & 0.129 & 0.677 & 0.760* & 0.760*\\
\bottomrule
\end{tabular}%
\end{table*}%

\subsection{Interpretability}
The output of Algorithm \ref{alg:mpc} is a set of cluster assignments and hyperplanes between clusters. To construct an explanation for a given cluster we include all hyperplanes related to the cluster to construct a polytope. We remove redundant hyperplanes (i.e., ones weaker than other constraints already included in the polytope). The flexibility of the MPC framework allows the resulting polytope explanation to resemble a number of different model classes. If we restrict the polytopes to only include axis-parallel hyperplanes (i.e., $M=\beta=1$), then each cluster example is a rule. A sample cluster explanation for the zoo dataset is:
$$
\text{(}\colorbox{lightgray}{HAS HAIR}\text{) AND (}\colorbox{lightgray}{HAS MILK}\text{) AND (}\colorbox{lightgray}{LEGS  $> 0$})
$$

Each hyperplane corresponds to a single clause in the resulting rule. Note that this rule could be equivalent to one leaf node in a decision tree, however the unordered conditions in a rule have been shown in a user study to be easier to interpret than a decision tree \citep{lakkaraju2016interpretable}. 

If we increase $\beta$, each explanation resembles a rule set where each hyperplane constitutes a rule. The zoo datasets contains primarily boolean features meaning that each hyperplane corresponds to a scorecard \cite{ustun2017optimized} where each condition has an associated weight which is compared against a threshold. One cluster explanation for the zoo dataset with $\beta=2, M = 3$ is: 

{\footnotesize
$$
\text{[(\colorbox{lightgray}{DOMESTIC}) + (\colorbox{lightgray}{HAS EGGS}) } > \text{1] AND (\colorbox{lightgray}{HAS TEETH})}
$$}

In this example the first rule in the rule set is a scorecard (both conditions need to be met). The rule sets require more effort to understand than a single rule, but still provide a clear explanation that can be understood by practitioners. For datasets with non-binerized features, the explanations remain rule sets but each rule is a more general linear condition. One cluster explanation for the wine dataset with $\beta=2, M = 3$ is:
{\footnotesize
\begin{align*}
\text{[3*(\colorbox{lightgray}{CITRIC ACIDITY})}&\text{ + 9*(\colorbox{lightgray}{DENSITY})} < \text{ 1.42]} \\
&\text{AND} \\
\text{[1*(\colorbox{lightgray}{pH LEVEL}) + } \text{2*} & \text{(\colorbox{lightgray}{CHLORIDES}) }\geq \text{0.58] }\\
\end{align*}}
Note that while the explanation is no longer easy to understand intuitively, it still remains audit-able which is sufficient for a number of applications. To improve the interpretability of explanations for datasets with real valued features, users can use binarized features for polytope construction (see Remark \ref{rem:features}). 
Another benefit of the MPC framework is that it provides pairwise comparisons between clusters. To compare clusters with a decision tree, users have to traverse a tree looking at different nodes which can be further complicated by tree-based explanations with multiple leaf nodes per cluster. In contrast, the hyperplane separating each pair of clusters acts as a pairwise comparison. A pairwise comparison between two clusters in the zoo dataset (for $M=\beta=1$) is:
$$
\text{IF (\colorbox{lightgray}{HAS HAIR})}: \text{Cluster 3 ELSE Cluster 4} 
$$

\section{Conclusion}
In this paper we introduce a novel algorithm for interpretable clustering that describes clusters using polytopes. We formulate the problem of jointly clustering and explaining clusters as a MINLP that optimizes both silhouette coefficient and representation error. A novel IP formulation for finding separating hyperplanes is used to enforce interpretability considerations on the resulting polytopes. To approximate a solution to the MINLP we leverage a two-phase optimization approach that first generates an initial set of clusters and polytopes using alternating minimization, then improves clustering performance using coordinate descent. Compared to state of the art uninterpretable and interpretable clustering algorithms our approach is able to find high quality clusters while preserving interpretability.

\bibliography{main.bib}


\appendix

\section{Proofs}
\begin{theorem}[K-Means Polytope Interpretability]
 Local solutions to the k-means clustering problem with euclidean distance can be perfectly separated from the other clusters by a polytope.
\end{theorem}

\begin{proof}
    Let $\mathcal{C} = \{C_1, C_2, \dots, C_k\}$ be a local solution to the k-means clustering problem. We start by proving that any two arbitrary clusters $C_i, C_j$ can be perfect separated by a linear hyperplane. Let $c^{i}, c^{j}$ be the cluster center of $C_i, C_j$ respectively. Local solutions must satisfy $x \in C_i$ $\|x - c^i \| < \|x - c^j \|$, otherwise simply changing the cluster of $x$ from $C_i$ to $C_j$ would lead to a lower objective. Without loss of generality we assume that all points that are equidistant from both centers are assigned to $C_i$. Note that the assignment with deterministic tie-breaking achieves the same $k$-means objective and thus is also a local minimum. Consider the following hyperplane $w^T x + b = 0$ defined by:
    
    $$
    w = c^j - c^i \quad\quad b = - (c^j - c^i)^T (\frac{c^i + c^j}{2}),
    $$
    
    We now propose that this line perfectly separates the two clusters. Without loss of generality assume that all points in $j$ lie above the plane (i.e. $w^T x + b \geq 0 ~~\forall x \in C_j$), and all points in $i$ lie below it (i.e. $w^T x + b \leq 0 \forall ~~x \in C_i$). Suppose this were not true, then there would exists a point $x \in C_i$ such that $w^T x + b > 0$. Since $x \in C_i$, we know that $\|x - c^i \| < \|x - c^j \|$, otherwise we would have a contradiction to the assignment being an output from the $k$-means algorithm. Using some simple algebra we get the following: 
    
    \begin{align*}
        w^T x + b &= (c_j - c_i)^T x - (c_j - c_i)^T (\frac{c_i + c_j}{2}) \\
        &= x^T c_j - x^T c_i - \frac{1}{2} (c_j)^T c_j + \frac{1}{2} (c_i)^T c_i  \\
        &= \frac{1}{2}(\|x - c_i\|^2 -  \|x - c_j\|^2)
    \end{align*}
    
    Note that this implies 
    $$
    w^T x + b > 0 \implies \|x - c_i\| \geq   \|x - c_j\|
    $$
    which contradicts $\|x - c^i \| < \|x - c^j \|$. An identical argument also works for $x \in C_j$ such that $w^T x + b < 0$, and thus the given hyperplane must divide the two clusters. 
    
    Now consider an arbitrary cluster $C_i$, and create a hyperplane between $C_i$ and every other cluster. The set of hyperplanes now defines a polytope that contains $C_i$ and separates it from the other clusters.
\end{proof}

\begin{theorem}[Alternating Minimization Improvement]
Algorithm 2 generates objective values for the representation aware k-means clustering problem that are monotonically decreasing for $l \geq 2$, and terminates in a finite number of iterations.
\end{theorem}
    \begin{proof}
    We start by showing that for $l \geq 2$, each loop of Algorithm 1 produces a monotonically decreasing objective value. After iteration $l = 1$ we have a feasible cluster assignment $z$ and a set of separating hyperplanes $w, b$. We will now show that after a single pass through the loop of Algorithm 1, the objective for Problem  (31)-(32) either decreases or remains the same triggering the end of the algorithm. 
    
    First consider the loop to adjust cluster assignments by solving Problem (31)-(32) (lines 4-8). Note that since we have a feasible assignment, the existing solution must be feasible to the updated version of Problem (31)-(33) with the given $\xi$ from the current separating hyperplanes. We start the clustering assignment portion of Algorithm 1 with the current assignment, thus solving (32)-(33) with the current $c_k$ must result in an equal or lower objective value. Based on the same logic as to the original k-means LLoyd's algorithm \cite{lloyd1982least}, updating $c_k$ is also guaranteed to maintain or decrease the objective value. By an identical argument every iteration of the clustering loop will similarly only maintain or decrease the current objective value, ensuring that the cluster assignment must result in an equal or lower objective value to the previous assignment. 
    
    Next we consider solving the separating hyperplane problem for clusters $i$ and $j$. Note that since each sub-problem only includes data points involved in cluster $i$ and $j$, the objective for problem (1)-(12) is equal to the contribution of objective term (26) for clusters $i$ and $j$ for the current assignment $z$. The resulting solution returned by solving the IP (1)-(12) must therefore result in a hyperplane with an objective term less than or equal to the current solution. Since this holds for the hyperplane between arbitrary clusters $i$ and $j$, it holds for every pair of clusters and thus the output of separating hyperplane loop results in an equal or lower objective. 
    
    Note that if the objective value remains constant, the algorithm terminates. Thus ensuring that Algorithm 1 leads to a monotonically decreasing series of objective values for problem (32)-(33). The finite termination condition follows from there being a finite number of possible cluster configurations and integral separating hyperplanes.

    \end{proof}

\section{Experiment Details}
We benchmarked the performance of MPC on two sets of data: (1) the fundamental clustering and projection suite (FCPS) \cite{ultsch2020fundamental}, and (2) a suite of UCI machine learning datasets \cite{Dua:2019}. The FCPS data was taken from the author's R library \footnote{\url{https://github.com/Mthrun/FCPS}}. The UCI data was taken from the UCI website. For each dataset we normalize each feature to be between 0 and 1 using a standard max-min rescaling function. For categorical features we convert them to real valued features using one hot encoding. We also remove any rows with missing values. 

All of our experiments were run on a personal laptop computer with a 2.7 GhZ Quad-Core Intel CPU, and 16 GB of RAM. To solve all the linear and integer programs in our framework we used CPLEX \cite{cplex2009v12} with the default parameters. To initialize our clustering we use the k-means implementation included in scikit-learn \cite{scikit-learn}. We set the random seed for all of our experiments to be 42. 

We benchmark the performance of MPC against the following suite of traditional and interpretable clustering algorithms: 
\begin{itemize}
    \item \textbf{K-means++:} We run K-means with the improved initialization scheme introduced in \cite{arthur2006k}. We use the K-means implementation included in scikit-learn with the k-means++ initialization \cite{scikit-learn}.
    \item \textbf{Gaussian Mixture Models (GMM):} GMM uses the expectation-minimization (EM) algorithm to fit a mixture of K Gaussian distributions to the data \cite{hastie2009unsupervised}. We use the scikit-learn implementation of the algorithm with a k-means initialization scheme and full covariance estimation \cite{scikit-learn}.
    \item \textbf{Hierarchical Clustering (HClust):} Hierarchal clustering is an agglomerative clustering method that constructs a dendogram that links all points together \cite{hastie2009unsupervised}. We use the scikit-learn implementation of the algorithm with ward linkage and euclidean distance \cite{scikit-learn}. 
    \item \textbf{Density-based Spatial Clustering of Applications with Noise (DBSCAN):} DBSCAN works by clustering points in high density regions together \cite{ester1996density}. However, this method does not necessarily assign every point to a cluster, and can instead leave points as outliers. To provide a fair comparison to other methods we assign each outlier point to the most common cluster in its 5 nearest neighbours with a label. This method also doesn't restrict the number of clusters identified, thus to provide a fair comparison we filter out solutions that have more than 10 clusters. We use the DBSCAN implementation in scikit-learn \cite{scikit-learn}, and tune the epsilon (neighbourhood size) parameter using a grid of 100 possible values between 0.1 and 5.
    \item \textbf{Interpretable Clustering via Optimal Decision Trees (ICOT):} ICOT builds upon the optimal decision tree framework \cite{bertsimas2017optimal} to perform clustering with decision trees by optimizing silhouette coefficient \cite{bertsimas2021interpretable}. We use the julia implementation of ICOT provided by the authors. We warm-start the trees using the optimal k-mean clustering with the optimal decision tree, set the maximum depth to 3, a minimum bucket size of 1, and use a geometric search threshold of 0.99.
    \item \textbf{Explainable K-Means Clustering (ExKMC):} ExKMC is an algorithm that uses decision trees to explain the output of k-means clustering \cite{frost2020exkmc}. We present two version of ExKMC results - one where we restrict each cluster to have a single leaf node (i.e. the most interpretable approach) which we dub ExKMC-1, and one set of results where we let each cluster have up to 5 associated leaf nodes, dubbed ExKMC-5. We use the python implementation of ExKMC provided by the authors.
\end{itemize}

We ran each algorithm 100 times and retained the best performing cluster.

\section{Computation Time}
For each  algorithm we report the average time to complete one run of each algorithm (averaged over the 100 random starts when applicable). For algorithms that require tuning a hyper-parameter (i.e. k) we report the computation time to find a solution for a single hyper-parameter setting.  Tables \ref{tab:fcps_comp} and \ref{tab:uci_comp} show the average computation time on the FCPS and UCI machine learning datasets respectively. Note that while MPC requires much more computational effort than traditional uninterpretable clustering algorithms, the computation time remains competitive with ICOT, neither strictly dominating the other with respect to computation time. 

\begin{table*}[h]
\centering\footnotesize
\caption{\label{tab:fcps_comp} Computation time in seconds on fundamental clustering and projection suite.}
\setlength{\tabcolsep}{5pt} 
\begin{tabular}{l c c c c c c c c c }		\toprule
Dataset &  $k$-means++ & GMM & HClust & DBSCAN & ICOT & ExKMC-1 & ExKMC-5 & \textbf{MPC-1} & \textbf{MPC-2} \\\midrule
Atom & 0.01 & 0.01 & 0.01 & 0.01 & 86.19 & 0.01 & 0.02 & 449.87 & 4304.88 \\
Chainlink & 0.01 & 0.01 & 0.01 & 0.01 & 183.62 & 0.01 & 0.02 & 255.58 & 3654.62 \\
Engytime & 0.01 & 0.02 & 0.23 & 0.03 & 3906.0 & 0.01 & 0.03 & 444.12 & 1358.17 \\
Hepta & 0.0 & 0.01 & 0.0 & 0.0 & 15.36 & 0.0 & 0.0 & 24.72 & 133.67 \\
Lsun & 0.0 & 0.01 & 0.0 & 0.0 & 31.1 & 0.0 & 0.01 & 34.53 & 367.46 \\
Target & 0.01 & 0.01 & 0.01 & 0.01 & 86.47 & 0.01 & 0.01 & 21.56 & 59.9 \\
Tetra & 0.0 & 0.0 & 0.0 & 0.0 & 22.64 & 0.0 & 0.0 & 27.02 & 110.77 \\
Wingnut & 0.0 & 0.01 & 0.02 & 0.01 & 87.8 & 0.01 & 0.01 & 16.95 & 68.38 \\
\bottomrule
\end{tabular}%
\end{table*}%

\begin{table*}[h]
\centering\footnotesize
\caption{\label{tab:uci_comp} Computation time in seconds on UCI Machine Learning Clustering problem sets.}
\setlength{\tabcolsep}{5pt} 
\begin{tabular}{l c c c c c c c c c }		\toprule
Dataset &$k$-means++ & GMM & HClust & DBSCAN & ICOT & ExKMC-1 & ExKMC-5 & \textbf{MPC-1} & \textbf{MPC-2} \\\midrule
Iris & 0.07 & 0.01 & 0.01 & 0.0 & 7.92 & 0.20 & 0.21 & 6.0 & 20.0 \\
Wine & 0.01 & 0.04 & 0.01 & 0.0 & 32.47 & 0.22 & 0.12 & 51.0 & 561.0 \\
Zoo & 0.09 & 0.02 & 0.01 & 0.0 & 15.49 & 0.21 & 0.29 & 24.0 & 143.0 \\
Seeds & 0.09 & 0.04 & 0.02 & 0.0 & 26.11 & 0.12 & 0.14 & 34.0 & 198.0 \\
Libras & 0.13 & 0.58 & 0.02 & 0.01 & 472.26 & 0.21 & 0.53 & 531.0 & 7100.0 \\
Framingham & 0.14 & 2.18 & 1.94 & 0.14 & 4367.0 & 0.30 & 0.30 & 3871.0 & 10244.0 \\
Bank & 0.44 & 0.65 & 8.41 & 0.31 & 25157.0 & 0.30 & 0.50 & 30211.0 & 101002.0 \\
Spam & 0.17 & 0.52 & 5.8 & 0.49 & 38371.0 & 0.90 & 0.50 & 2683.0 & 8012.0 \\
\bottomrule
\end{tabular}%
\end{table*}%

\section{Number of Clusters}
Tables \ref{tab:fcps_k} and \ref{tab:uci_k} show the optimal number of clusters found for each of the clustering algorithms evaluated. 

\begin{table*}[h!]
\centering\footnotesize
\caption{\label{tab:fcps_k} Optimal number of clusters for each algorithm on fundamental clustering and projection suite.}
\setlength{\tabcolsep}{5pt} 
\begin{tabular}{l c c c c c c c c c c }		\toprule
Dataset & $k$-means++ & GMM & HClust & DBSCAN & ICOT & ExKMC-1 & ExKMC-5 & \textbf{MPC-1} & \textbf{MPC-2} \\\midrule
Atom & 10 & 10 & 10 & 10 & 8 & 9 & 10 & 10 & 9 \\
Chainlink & 10 & 8 & 10 & 7 & 2 & 10 & 8 & 10 & 10 \\
Engytime & 3 & 2 & 3 & 2 & 2 & 3 & 3 & 2 & 2 \\
Hepta & 7 & 7 & 7 & 7 & 4 & 7 & 7 & 7 & 7 \\
Lsun & 6 & 6 & 5 & 3 & 4 & 4 & 5 & 3 & 4 \\
Target & 9 & 7 & 6 & 6 & 2 & 8 & 9 & 2 & 2 \\
Tetra & 4 & 4 & 4 & 4 & 4 & 4 & 4 & 4 & 4 \\
Wingnut & 2 & 2 & 2 & 8 & 2 & 2 & 2 & 2 & 2 \\
\bottomrule
\end{tabular}%
\end{table*}%

\begin{table*}[h!]
\centering\footnotesize
\caption{\label{tab:uci_k} Optimal number of clusters for each algorithm on UCI machine learning datasets.}
\setlength{\tabcolsep}{5pt} 
\begin{tabular}{l c c c c c c c c c }		\toprule
Dataset & $k$-means++ & GMM & HClust & DBSCAN & ICOT & ExKMC-1 & ExKMC-5 & \textbf{MPC-1} & \textbf{MPC-2} \\\midrule
Iris & 2 & 2 & 2 & 2 & 2 & 2 & 2 & 2 & 2 \\
Wine & 2 & 2 & 2 & 2 & 2 & 3 & 2 & 2 & 2 \\
Zoo & 4 & 5 & 6 & 5 & 5 & 4 & 6 & 5 & 5 \\
Seeds & 2 & 2 & 2 & 4 & 2 & 2 & 2 & 2 & 2 \\
Libras & 10 & 10 & 8 & 10 & 8 & 2 & 10 & 10 & 10 \\
Framingham & 8 & 9 & 10 & 8 & 8 & 9 & 9 & 10 & 10 \\
Bank & 7 & 2 & 2 & 3 & 2 & 2 & 2 & 2 & 2 \\
Spam & 2 & 2 & 2 & 2 & 2 & 6 & 2 & 2 & 2 \\
\bottomrule
\end{tabular}%
\end{table*}%

\end{document}